\documentclass[11pt]{article}

\usepackage[utf8]{inputenc}
\usepackage[T1]{fontenc}
\usepackage{lmodern}
\usepackage{amsfonts}
\usepackage{mathtools}
\usepackage{graphicx}
\usepackage{booktabs}
\usepackage{array}
\usepackage{adjustbox}
\usepackage{enumitem}
\usepackage{xcolor}
\usepackage{url}

\usepackage[letterpaper,margin=1.05in]{geometry}
\usepackage{times}
\usepackage[round]{natbib}
\PassOptionsToPackage{colorlinks=true,linkcolor=blue,citecolor=blue,urlcolor=blue}{hyperref}

\usepackage{bm}

\usepackage{tabularx}
\usepackage{placeins}

\usepackage{hyperref}

\usepackage{cleveref}

\graphicspath{{figures/}}
\hypersetup{
    colorlinks=true,
    linkcolor=blue,
    citecolor=blue,
    urlcolor=blue,
    pdftitle={\bfseries Whitening Inverts the Hierarchy:\\
What the Norm of a Whitened Embedding Measures},
    pdfauthor={Mohammed AHNOUCH and Lotfi ELAACHAK}
}

\newif\ifshort \shortfalse
\newif\ifarxiv \arxivtrue

\usepackage[T1]{fontenc}
\usepackage[utf8]{inputenc}
\usepackage[expansion=false]{microtype}
\usepackage{amsmath,amssymb,amsthm}
\usepackage{booktabs}
\usepackage{array}
\usepackage{multirow}
\usepackage{tikz}
\usepackage{pgfplots}
\usepackage{caption}
\usepackage{subcaption}
\usepackage{hyperref}
\usepackage{url}
\usepackage{placeins}
\usepackage{cleveref}

\pgfplotsset{compat=1.18}
\usetikzlibrary{arrows.meta,positioning,calc,fit,backgrounds}

\newcommand{\ClIBdcorr}{-0.413}
\newcommand{\ClIBdcv}{0.0442}
\newcommand{\ClIBdd}{512}

\newcommand{\ClIBdkpca}{+1.494}
\newcommand{\ClIBdkraw}{-0.012}
\newcommand{\ClIBdover}{6.62}

\newcommand{\ClIBlcorr}{-0.421}
\newcommand{\ClIBlcv}{0.0505}
\newcommand{\ClIBld}{512}

\newcommand{\ClIBlkpca}{+1.438}
\newcommand{\ClIBlkraw}{-0.008}
\newcommand{\ClIBlover}{6.64}

\newcommand{\ClIBocorr}{-0.435}
\newcommand{\ClIBocv}{0.0528}
\newcommand{\ClIBod}{512}

\newcommand{\ClIBokpca}{+1.523}
\newcommand{\ClIBokraw}{+0.009}
\newcommand{\ClIBoover}{6.30}

\newcommand{\ClIDicorr}{-0.505}
\newcommand{\ClIDicv}{0.0199}
\newcommand{\ClIDid}{768}

\newcommand{\ClIDikpca}{+0.687}
\newcommand{\ClIDikraw}{+0.007}
\newcommand{\ClIDiover}{6.42}

\newcommand{\ClILcorr}{-0.453}
\newcommand{\ClILcv}{0.0444}
\newcommand{\ClILd}{768}

\newcommand{\ClILkpca}{+0.542}
\newcommand{\ClILkraw}{+0.022}
\newcommand{\ClILover}{5.90}

\newcommand{\ClISicorr}{-0.512}
\newcommand{\ClISicv}{0.0710}
\newcommand{\ClISid}{768}

\newcommand{\ClISikpca}{+0.922}
\newcommand{\ClISikraw}{+0.039}
\newcommand{\ClISiover}{7.55}

\newcommand{\Dim}{768}
\newcommand{\MechBV}{0.312}
\newcommand{\MechBVone}{0.223}

\newcommand{\MechBhalves}{+0.41}

\newcommand{\MechBhsraw}{67}
\newcommand{\MechBhswh}{13}
\newcommand{\MechBk}{68}

\newcommand{\MechBkt}{1.37}
\newcommand{\MechBktp}{0.94}

\newcommand{\MechBrho}{0.42}

\newcommand{\MechBtailvar}{99}
\newcommand{\MechBtvr}{50.7}
\newcommand{\MechBvbias}{1.40}
\newcommand{\MechDV}{0.164}
\newcommand{\MechDVone}{0.161}

\newcommand{\MechDeff}{205}
\newcommand{\MechDhalves}{+0.91}

\newcommand{\MechDhsraw}{69}
\newcommand{\MechDhswh}{27}
\newcommand{\MechDk}{205}

\newcommand{\MechDkt}{0.45}
\newcommand{\MechDktp}{0.49}

\newcommand{\MechDrho}{0.96}

\newcommand{\MechDtailvar}{82}
\newcommand{\MechDtvr}{46.4}

\newcommand{\MechGhalves}{-0.08}

\newcommand{\MechGk}{94}

\newcommand{\MechGkt}{-0.01}

\newcommand{\MechGtailvar}{91}
\newcommand{\MechImV}{0.137}
\newcommand{\MechImVone}{0.123}

\newcommand{\MechImeff}{95}
\newcommand{\MechImhalves}{+0.76}

\newcommand{\MechImhsraw}{66}
\newcommand{\MechImhswh}{12}
\newcommand{\MechImk}{95}

\newcommand{\MechImkt}{0.49}
\newcommand{\MechImktp}{0.41}

\newcommand{\MechImrho}{0.79}

\newcommand{\MechImtailvar}{96}
\newcommand{\MechImtvr}{42.5}

\newcommand{\MechSV}{0.211}
\newcommand{\MechSVone}{0.183}

\newcommand{\MechShalves}{+0.71}

\newcommand{\MechShsraw}{66}
\newcommand{\MechShswh}{10}
\newcommand{\MechSk}{73}

\newcommand{\MechSkt}{0.77}
\newcommand{\MechSktp}{0.63}

\newcommand{\MechSrho}{0.73}

\newcommand{\MechStailvar}{99}
\newcommand{\MechStvr}{64.7}

\newcommand{\MechTxV}{0.221}
\newcommand{\MechTxVone}{0.207}

\newcommand{\MechTxhalves}{+0.85}

\newcommand{\MechTxhsraw}{69}
\newcommand{\MechTxhswh}{9}
\newcommand{\MechTxk}{67}

\newcommand{\MechTxkt}{0.70}
\newcommand{\MechTxktp}{0.66}

\newcommand{\MechTxrho}{0.88}

\newcommand{\MechTxtailvar}{101}
\newcommand{\MechTxtvr}{73.9}

\newcommand{\Npub}{5{,}000}
\newcommand{\VMax}{0.31}
\newcommand{\VMin}{0.14}
\newcommand{\actIm}{31}
\newcommand{\actTx}{33}

\newcommand{\adZca}{0.4366}

\newcommand{\adavgPubIm}{0.4890}
\newcommand{\adavgPubTx}{0.5926}

\newcommand{\adpassIm}{98.3}
\newcommand{\adpassTx}{90.1}

\newcommand{\bppMax}{-0.04}
\newcommand{\bppMin}{-0.09}

\newcommand{\calAfiveIm}{35.8\%}
\newcommand{\calAfiveTx}{36.5\%}
\newcommand{\calAoneIm}{31.2\%}
\newcommand{\calAoneTx}{33.0\%}
\newcommand{\calApointIm}{26.5\%}
\newcommand{\calApointTx}{29.3\%}
\newcommand{\calAtenIm}{38.3\%}
\newcommand{\calAtenTx}{38.4\%}
\newcommand{\capMax}{+0.13}
\newcommand{\capMin}{+0.04}
\newcommand{\capSig}{9}

\newcommand{\confRadIm}{-0.66}

\newcommand{\confRawIm}{+0.86}

\newcommand{\corpAcross}{+0.14}
\newcommand{\corpAcrossP}{0.79}
\newcommand{\corpClCifarh}{5.75}
\newcommand{\corpClCoco}{5.54}
\newcommand{\corpClDtd}{6.02}
\newcommand{\corpClFlowers}{6.47}
\newcommand{\corpClPets}{6.72}
\newcommand{\corpClSvhn}{8.11}
\newcommand{\corpDiCifarh}{4.18}
\newcommand{\corpDiCoco}{6.18}
\newcommand{\corpDiDtd}{7.82}
\newcommand{\corpDiFlowers}{6.13}
\newcommand{\corpDiPets}{6.08}
\newcommand{\corpDiSvhn}{7.62}
\newcommand{\crossGa}{-10}
\newcommand{\crossGb}{+42}
\newcommand{\crossGc}{+98}
\newcommand{\crossRa}{-10}
\newcommand{\crossRb}{+30}
\newcommand{\crossRc}{+53}
\newcommand{\cvBand}{3.56}
\newcommand{\cvBetweenCl}{0.929}
\newcommand{\cvBetweenDi}{1.312}
\newcommand{\cvFCl}{156}
\newcommand{\cvFDi}{154}
\newcommand{\cvMax}{0.0710}

\newcommand{\cvPCl}{1.7\times 10^{-5}}

\newcommand{\cvRatioCl}{8.8}

\newcommand{\cvWithinCl}{0.105}
\newcommand{\cvWithinDi}{0.149}
\newcommand{\devMeanPubIm}{0.98}
\newcommand{\devMeanPubTx}{2.85}

\newcommand{\devSdPubIm}{0.55}
\newcommand{\devSdPubTx}{13.24}
\newcommand{\devSdReIm}{0.23}
\newcommand{\devSdReTx}{13.45}
\newcommand{\dpavgPubIm}{0.3624}
\newcommand{\dpavgPubTx}{0.2568}
\newcommand{\dpnull}{0.5052}
\newcommand{\dpnullsd}{0.0021}
\newcommand{\dppassIm}{99.3}
\newcommand{\dppassTx}{99.2}

\newcommand{\drvCtail}{+0.50}
\newcommand{\drvCtailCl}{+0.93}

\newcommand{\drvCtailP}{1.8\times 10^{-2}}
\newcommand{\drvEff}{-0.60}
\newcommand{\drvEffCl}{-0.92}

\newcommand{\drvEffP}{3.4\times 10^{-3}}
\newcommand{\drvTcor}{+0.68}
\newcommand{\drvTcorP}{5.6\times 10^{-4}}

\newcommand{\gpub}{250}

\newcommand{\hsRawMax}{69}
\newcommand{\hsRawMin}{66}
\newcommand{\hsWhMax}{27}
\newcommand{\hsWhMin}{9}

\newcommand{\kappaADfive}{0.01508}

\newcommand{\kmeanIm}{+0.441}
\newcommand{\kmeanTx}{+0.658}
\newcommand{\knnClipL}{0.84}
\newcommand{\knnCons}{0.83}
\newcommand{\knnDiag}{0.80}
\newcommand{\knnMax}{0.93}
\newcommand{\knnMin}{0.70}

\newcommand{\krmsIm}{0.576}
\newcommand{\krmsTx}{0.749}

\newcommand{\lossVal}{35.766749}

\newcommand{\mAlladpassGau}{96.4}
\newcommand{\mAlladpassIm}{26.4}
\newcommand{\mAlladpassTx}{2.6}

\newcommand{\mGrpadavgGau}{0.3885}
\newcommand{\mGrpadavgIm}{0.4951}
\newcommand{\mGrpadavgTx}{0.5802}
\newcommand{\mGrpadpassGau}{100.0}
\newcommand{\mGrpadpassIm}{98.7}
\newcommand{\mGrpadpassTx}{92.3}
\newcommand{\mGrpdpavgGau}{0.5066}
\newcommand{\mGrpdpavgIm}{0.3630}
\newcommand{\mGrpdpavgTx}{0.2741}

\newcommand{\mGrpdppassIm}{99.3}
\newcommand{\mGrpdppassTx}{99.2}

\newcommand{\mIdentPctTx}{-1.07}
\newcommand{\mIdentTx}{759.8}

\newcommand{\mOverIm}{6.04}

\newcommand{\mRadIm}{27.43}
\newcommand{\mRadTx}{27.07}

\newcommand{\mSdIm}{3.92}
\newcommand{\mSdTx}{5.18}
\newcommand{\monoMargin}{37}
\newcommand{\muTheo}{27.7038}
\newcommand{\muzero}{0.3891}
\newcommand{\muzerosd}{0.0016}
\newcommand{\nBlock}{2{,}500}
\newcommand{\nCorpBlocks}{22}

\newcommand{\nstarConst}{216}
\newcommand{\nstarFacIm}{4.5}
\newcommand{\nstarFacTx}{2.0}

\newcommand{\nstarMIm}{1{,}113}
\newcommand{\nstarMTx}{499}

\newcommand{\nullreps}{24}

\newcommand{\oodFlagDC}{61}

\newcommand{\oodFprDC}{13}

\newcommand{\oodHeadCC}{0.414}
\newcommand{\oodHeadCF}{0.801}
\newcommand{\oodHeadDC}{0.167}
\newcommand{\oodHeadDF}{0.212}
\newcommand{\oodKnnCC}{0.688}
\newcommand{\oodKnnCF}{0.895}
\newcommand{\oodKnnDC}{0.957}
\newcommand{\oodKnnDF}{0.939}
\newcommand{\oodMahCC}{0.895}
\newcommand{\oodMahCF}{0.947}
\newcommand{\oodMahDC}{0.974}
\newcommand{\oodMahDF}{0.907}
\newcommand{\oodMppCC}{0.853}
\newcommand{\oodMppCF}{0.944}
\newcommand{\oodMppDC}{0.973}
\newcommand{\oodMppDF}{0.954}
\newcommand{\oodRawCC}{0.159}
\newcommand{\oodRawCF}{0.769}
\newcommand{\oodRawDC}{0.367}
\newcommand{\oodRawDF}{0.508}

\newcommand{\oodSpRLDC}{0.99}

\newcommand{\oodTailCC}{0.896}
\newcommand{\oodTailCF}{0.949}
\newcommand{\oodTailDC}{0.973}
\newcommand{\oodTailDF}{0.916}

\newcommand{\oodTcLtDC}{0.88}

\newcommand{\oodTcRawDC}{0.88}

\newcommand{\overBand}{1.28}
\newcommand{\overMax}{7.55}

\newcommand{\overMin}{5.90}
\newcommand{\overTrueIm}{5.57}
\newcommand{\overTrueTx}{8.09}
\newcommand{\pmaIm}{99.2}
\newcommand{\pmaTx}{99.2}
\newcommand{\pmbIm}{98.7}
\newcommand{\pmbTx}{92.3}
\newcommand{\pmcIm}{90.8}
\newcommand{\pmcTx}{57.3}
\newcommand{\pmdIm}{71.1}
\newcommand{\pmdTx}{31.0}
\newcommand{\pmeIm}{40.2}
\newcommand{\pmeTx}{10.8}
\newcommand{\pmfIm}{26.4}
\newcommand{\pmfTx}{2.6}
\newcommand{\ppaIm}{97.3}
\newcommand{\ppaTx}{96.6}
\newcommand{\ppbIm}{89.1}
\newcommand{\ppbTx}{81.0}
\newcommand{\ppcIm}{79.0}
\newcommand{\ppcTx}{49.5}
\newcommand{\ppdIm}{64.6}
\newcommand{\ppdTx}{32.7}
\newcommand{\ppeIm}{40.0}
\newcommand{\ppeTx}{13.9}
\newcommand{\ppfIm}{27.0}
\newcommand{\ppfTx}{6.0}
\newcommand{\pthra}{1.32}
\newcommand{\pthrb}{0.93}
\newcommand{\pthrc}{0.66}
\newcommand{\pthrd}{0.47}
\newcommand{\pthre}{0.29}
\newcommand{\pthrf}{0.21}
\newcommand{\pubMeanIm}{27.43}
\newcommand{\pubMeanTheo}{27.7}
\newcommand{\pubMeanTx}{28.49}
\newcommand{\pubSdIm}{3.94}
\newcommand{\pubSdTheoIm}{3.96}
\newcommand{\pubSdTheoTx}{6.60}
\newcommand{\pubSdTx}{5.72}

\newcommand{\rHiIm}{0.00}

\newcommand{\rKeepIm}{60}
\newcommand{\rKeepTx}{78}
\newcommand{\rLoIm}{3.65}

\newcommand{\rSpearIm}{0.836}
\newcommand{\rSpearTx}{0.947}
\newcommand{\radAfter}{140.4}
\newcommand{\radBefore}{0.67}

\newcommand{\rawknnMax}{-0.09}
\newcommand{\rawknnMin}{-0.66}
\newcommand{\recAfiveIm}{5.0\%}
\newcommand{\recAfiveTx}{5.1\%}
\newcommand{\recAoneIm}{1.0\%}
\newcommand{\recAoneTx}{1.1\%}
\newcommand{\recApointIm}{0.1\%}
\newcommand{\recApointTx}{0.1\%}
\newcommand{\recAtenIm}{10.1\%}
\newcommand{\recAtenTx}{10.0\%}
\newcommand{\refSd}{0.649}
\newcommand{\refSeeds}{30}
\newcommand{\repImMax}{1.2}

\newcommand{\repTxRadMean}{5.0}
\newcommand{\repTxRadSd}{9.4}
\newcommand{\repTxTestMax}{2.4}
\newcommand{\rhoMax}{0.96}
\newcommand{\rhoMin}{0.42}
\newcommand{\sCorrDatIm}{-0.4533 \pm 0.0229}
\newcommand{\sCorrDatTx}{-0.2698 \pm 0.0122}
\newcommand{\sCorrGclIm}{+0.0036 \pm 0.0110}
\newcommand{\sCorrGclTx}{+0.0050 \pm 0.0123}

\newcommand{\sCorrSclIm}{-0.0005 \pm 0.0179}
\newcommand{\sCorrSclTx}{-0.0183 \pm 0.0180}
\newcommand{\sKpcaDatIm}{+0.5423 \pm 0.0516}
\newcommand{\sKpcaDatTx}{+0.9532 \pm 0.0609}
\newcommand{\sKpcaGclIm}{-0.0020 \pm 0.0028}
\newcommand{\sKpcaGclTx}{-0.0024 \pm 0.0026}
\newcommand{\sKpcaSclIm}{+0.2514 \pm 0.0149}
\newcommand{\sKpcaSclTx}{+0.4354 \pm 0.0136}
\newcommand{\sKrawDatIm}{+0.0216 \pm 0.0031}
\newcommand{\sKrawDatTx}{-0.0116 \pm 0.0046}
\newcommand{\sKrawGclIm}{+0.0025 \pm 0.0045}
\newcommand{\sKrawGclTx}{-0.0026 \pm 0.0053}
\newcommand{\sKrawSclIm}{+0.2476 \pm 0.0132}
\newcommand{\sKrawSclTx}{+0.4274 \pm 0.0186}
\newcommand{\sRsdDatIm}{5.99}
\newcommand{\sRsdDatTx}{8.09}
\newcommand{\sRsdGclIm}{1.02}
\newcommand{\sRsdGclTx}{1.01}
\newcommand{\sRsdSclIm}{4.78}
\newcommand{\sRsdSclTx}{6.39}
\newcommand{\sdTheo}{0.7070}
\newcommand{\sdchi}{0.707}

\newcommand{\shareMax}{0.87}
\newcommand{\shareMin}{0.39}
\newcommand{\shareRad}{0.62}
\newcommand{\shareRaw}{0.18}

\newcommand{\statNull}{0.014}
\newcommand{\subRootIm}{+15.60}
\newcommand{\subRootTx}{-43.68}
\newcommand{\subSdIm}{3.949}
\newcommand{\subSdTx}{6.609}
\newcommand{\svhnEff}{9}
\newcommand{\svhnTail}{98}
\newcommand{\swSpear}{+0.37}
\newcommand{\swSpearP}{0.47}
\newcommand{\swStrongAny}{64}
\newcommand{\swWeakAny}{8}
\newcommand{\swWeakCorr}{16}
\newcommand{\swWeakKraw}{8}
\newcommand{\swWeakRsd}{18}
\newcommand{\szCorrGIm}{-18.0}

\newcommand{\szCorrSIm}{-15.6}

\newcommand{\szKrawSIm}{-16.6}
\newcommand{\szKrawSTx}{-22.9}
\newcommand{\tailIm}{31}
\newcommand{\tailTx}{50}
\newcommand{\tailvarMax}{101}
\newcommand{\tailvarMin}{82}

\newcommand{\trioChi}{0.96}
\newcommand{\trioP}{0.62}
\newcommand{\trioPooled}{6.49}
\newcommand{\trioRhi}{37.6}
\newcommand{\trioRlo}{32.3}
\newcommand{\zadIm}{+62}
\newcommand{\zadTx}{+126}
\newcommand{\zdpIm}{-67}
\newcommand{\zdpTx}{-117}

\newcommand{\R}{\mathbb{R}}
\newcommand{\Ex}{\mathbb{E}}
\newcommand{\Nn}{\mathcal{N}}
\newcommand{\Id}{\mathrm{I}}
\newcommand{\Sph}{\mathbb{S}}
\newcommand{\KL}{\mathrm{KL}}
\newcommand{\unif}{\mathrm{Unif}}
\newcommand{\dx}{\,\mathrm{d}}
\newcommand{\norm}[1]{\lVert #1 \rVert}
\newcommand{\inner}[2]{\langle #1, #2 \rangle}
\newcommand{\Ad}{A^2}
\newcommand{\corr}{\mathrm{corr}}
\newcommand{\Prob}{\mathbb{P}}
\newcommand{\Var}{\mathrm{Var}}
\newcommand{\Cov}{\mathrm{Cov}}
\newcommand{\tr}{\mathrm{tr}}

\theoremstyle{plain}
\newtheorem{theorem}{Theorem}
\newtheorem{proposition}[theorem]{Proposition}
\newtheorem{lemma}[theorem]{Lemma}

\pgfplotsset{
  every axis/.append style={
    font=\footnotesize, line width=0.6pt, tick style={line width=0.5pt},
    grid=major, grid style={gray!22, line width=0.3pt},
    axis line style={gray!55},
  },
  legend style={font=\scriptsize, draw=gray!50, fill=white, fill opacity=0.9,
                text opacity=1, inner sep=2pt, row sep=1pt},
}

\definecolor{frClipL}{RGB}{31,119,180}
\definecolor{frClipB}{RGB}{214,120,28}
\definecolor{frDino}{RGB}{44,140,76}
\definecolor{frClone}{RGB}{120,120,120}
\definecolor{frData}{RGB}{178,54,54}
\definecolor{frDir}{RGB}{20,120,130}
\definecolor{frRad}{RGB}{180,90,40}
\ifarxiv
\pgfplotsset{
  every axis/.append style={
    font=\small, line width=0.8pt, tick style={line width=0.6pt},
    grid=major, grid style={gray!18, line width=0.35pt},
    axis line style={gray!45},
    tick label style={font=\footnotesize},
    label style={font=\small},
  },
  legend style={font=\footnotesize, draw=gray!40, fill=white, fill opacity=0.92,
                text opacity=1, inner sep=3pt, row sep=1.5pt},
}
\fi

\newcommand{\blockEqEight}{%
\begin{equation}
  \Ex[S] = \sqrt2\,\frac{\Gamma((d+1)/2)}{\Gamma(d/2)},
  \qquad
  \mathrm{Std}(S) = \sqrt{d - \mu_S^2}\,,
  \label{eq:eight}
\end{equation}
}

\newcommand{\blockSubstTab}{%
\begin{center}\small
\begin{tabular}{lcccccc}
\toprule
 & $\hat m$ & $\sqrt{|d-\hat m^2|}$ & published ``theory'' & \multicolumn{2}{c}{published deviation}
 & recomputed \\
\cmidrule(lr){5-6}
 & & & & mean & sd & sd \\
\midrule
images   & $\pubMeanIm$ & $\sqrt{\subRootIm} = \subSdIm$ & $\pubSdTheoIm$
         & $\devMeanPubIm\%$ & $\devSdPubIm\%$ & $\devSdReIm\%$ \\
captions & $\pubMeanTx$ & $\sqrt{\subRootTx} = \subSdTx$ & $\pubSdTheoTx$
         & $\devMeanPubTx\%$ & $\devSdPubTx\%$ & $\devSdReTx\%$ \\
\bottomrule
\end{tabular}
\end{center}
}

\newcommand{\blockPassTab}{%
\begin{center}\small
\begin{tabular}{rccccc}
\toprule
group size & threshold $|\gamma_2|$ & \multicolumn{2}{c}{images} & \multicolumn{2}{c}{captions} \\
\cmidrule(lr){3-4}\cmidrule(lr){5-6}
 & & predicted & measured & predicted & measured \\
\midrule
$125$   & $\pthra$ & $\ppaIm\%$ & $\pmaIm\%$ & $\ppaTx\%$ & $\pmaTx\%$ \\
$250$   & $\pthrb$ & $\ppbIm\%$ & $\pmbIm\%$ & $\ppbTx\%$ & $\pmbTx\%$ \\
$500$   & $\pthrc$ & $\ppcIm\%$ & $\pmcIm\%$ & $\ppcTx\%$ & $\pmcTx\%$ \\
$1000$  & $\pthrd$ & $\ppdIm\%$ & $\pmdIm\%$ & $\ppdTx\%$ & $\pmdTx\%$ \\
$2500$  & $\pthre$ & $\ppeIm\%$ & $\pmeIm\%$ & $\ppeTx\%$ & $\pmeTx\%$ \\
$5000$  & $\pthrf$ & $\ppfIm\%$ & $\pmfIm\%$ & $\ppfTx\%$ & $\pmfTx\%$ \\
\bottomrule
\end{tabular}
\end{center}
}

\newcommand{\blockZTab}{%
\begin{table}[htbp]
\centering\small
\caption{Published grand averages against the null law of the protocol itself, obtained by
simulating the entire protocol on exact Gaussians $\nullreps$ times, with our measured values
alongside.}
\label{tab:z}
\begin{tabular}{llccrc}
\toprule
statistic & modality & published & exact Gaussian null & $z$ & measured \\
\midrule
mean Anderson--Darling $\Ad$ & images   & $\adavgPubIm$ & $\muzero \pm \muzerosd$ & $\zadIm$ & $\mGrpadavgIm$ \\
                             & captions & $\adavgPubTx$ & $\muzero \pm \muzerosd$ & $\zadTx$ & $\mGrpadavgTx$ \\
mean D'Agostino--Pearson $p$ & images   & $\dpavgPubIm$ & $\dpnull \pm \dpnullsd$ & $\zdpIm$ & $\mGrpdpavgIm$ \\
                             & captions & $\dpavgPubTx$ & $\dpnull \pm \dpnullsd$ & $\zdpTx$ & $\mGrpdpavgTx$ \\
\bottomrule
\end{tabular}
\end{table}
}

\newcommand{\blockCorpusTab}{%
\begin{table}[htbp]
\centering\small
\caption{Over-dispersion by corpus, two encoders. The within-corpus figure is the standard deviation
of a single block, estimated from disjoint pairs of the same distribution.}
\label{tab:corpus}
\begin{tabular}{lcccccc c}
\toprule
 & COCO & CIFAR-100 & pets & flowers & DTD & SVHN & within / between \\
\midrule
CLIP ViT-B/32 & \corpClCoco$\times$ & \corpClCifarh$\times$ & \corpClPets$\times$
  & \corpClFlowers$\times$ & \corpClDtd$\times$ & \corpClSvhn$\times$
  & $\cvWithinCl$ / $\cvBetweenCl$ \\
DINOv2 ViT-B/14 & \corpDiCoco$\times$ & \corpDiCifarh$\times$ & \corpDiPets$\times$
  & \corpDiFlowers$\times$ & \corpDiDtd$\times$ & \corpDiSvhn$\times$
  & $\cvWithinDi$ / $\cvBetweenDi$ \\
\bottomrule
\end{tabular}
\end{table}
}

\newcommand{\blockCalibTab}{%
\begin{table}[htbp]
\centering\small
\caption{Actual exceedance of a threshold set from the Gaussian surrogate, and from the
one-dimensional recalibration \cref{eq:repair}, when the truth is the measured radial law.}
\label{tab:calib}
\begin{tabular}{lcccc}
\toprule
nominal & \multicolumn{2}{c}{images} & \multicolumn{2}{c}{captions} \\
\cmidrule(lr){2-3}\cmidrule(lr){4-5}
 & Gaussian & recalibrated & Gaussian & recalibrated \\
\midrule
$10\%$   & $\calAtenIm$   & $\recAtenIm$   & $\calAtenTx$   & $\recAtenTx$ \\
$5\%$    & $\calAfiveIm$  & $\recAfiveIm$  & $\calAfiveTx$  & $\recAfiveTx$ \\
$1\%$    & $\calAoneIm$   & $\recAoneIm$   & $\calAoneTx$   & $\recAoneTx$ \\
$0.1\%$  & $\calApointIm$ & $\recApointIm$ & $\calApointTx$ & $\recApointTx$ \\
\bottomrule
\end{tabular}
\end{table}
}

\newcommand{\blockPolarFig}{%
\begin{figure}[t]
\centering
\ifarxiv
\begin{tikzpicture}[
   scale=0.95, transform shape,
   font=\footnotesize,
   box/.style={draw=gray!55, rounded corners=2.5pt, minimum height=8.5mm,
               inner xsep=5pt, align=center, fill=white},
   ar/.style={-{Latex[length=2.2mm]}, gray!65, line width=0.7pt}]
 \node[box] (x) {input $x$};
 \node[box, right=4.5mm of x] (f) {encoder};
 \node[box, right=4.5mm of f] (z) {$z \in \R^{d}$};
 \node[box, right=4.5mm of z] (w) {$w = \Sigma^{-1/2}(z-\mu)$};
 \draw[ar] (x)--(f); \draw[ar] (f)--(z); \draw[ar] (z)--(w);
 \node[box, right=6.5mm of w, fill=frDir!12, draw=frDir!55] (u)
   {direction\\$u = w/\norm{w}$};
 \node[box, below=4.5mm of u, fill=frRad!12, draw=frRad!55] (r)
   {radius\\$R = \norm{w}$};
 \draw[ar, frDir!70] (w.east)-- ++(3mm,0) |- (u.west);
 \draw[ar, frRad!70] (w.east)-- ++(3mm,0) |- (r.west);
 \node[right=4mm of u, align=left, text width=3.55cm] (ul)
   {\itshape what the loss sees: a contrastive objective is a function of $u$ alone};
 \node[right=4mm of r, align=left, text width=3.55cm] (rl)
   {\itshape what the surrogate uses: $\tfrac12 R^2$, which the loss never constrains};
 \draw[ar, frDir!60] (u)--(ul); \draw[ar, frRad!60] (r)--(rl);
 \begin{scope}[on background layer]
   \node[fit=(u)(r), draw=gray!40, densely dashed, rounded corners=4pt,
         inner sep=4.5pt, fill=gray!3] {};
 \end{scope}
\end{tikzpicture}
\else
\begin{tikzpicture}[
   scale=0.92, transform shape,
   font=\footnotesize,
   box/.style={draw=gray!60, rounded corners=2pt, minimum height=8mm, inner xsep=5pt, align=center},
   ar/.style={-{Latex[length=2mm]}, gray!70, line width=0.6pt}]
 \node[box] (x) {input $x$};
 \node[box, right=5mm of x] (f) {encoder};
 \node[box, right=5mm of f] (z) {$z \in \R^{d}$};
 \node[box, right=5mm of z] (w) {$w = \Sigma^{-1/2}(z-\mu)$};
 \draw[ar] (x)--(f); \draw[ar] (f)--(z); \draw[ar] (z)--(w);
 \node[box, right=7mm of w, fill=gray!8] (u) {direction\\$u = w/\norm{w}$};
 \node[box, below=4mm of u, fill=gray!8] (r) {radius\\$R = \norm{w}$};
 \draw[ar] (w.east)-- ++(3mm,0) |- (u.west);
 \draw[ar] (w.east)-- ++(3mm,0) |- (r.west);
 \node[right=5mm of u, align=left, text width=3.4cm] (ul)
   {\itshape what the loss sees: a contrastive objective is a function of $u$ alone};
 \node[right=5mm of r, align=left, text width=3.4cm] (rl)
   {\itshape what the surrogate uses: $\tfrac12 R^2$, which the loss never constrains};
 \draw[ar] (u)--(ul); \draw[ar] (r)--(rl);
 \begin{scope}[on background layer]
   \node[fit=(u)(r), draw=gray!35, dashed, rounded corners=3pt, inner sep=4pt] {};
 \end{scope}
\end{tikzpicture}
\fi
\caption{The polar decomposition of a whitened embedding, and the division this paper is about. The
directional factor is what a contrastive loss optimises; the radial factor is the whole of the
likelihood surrogate, and \Cref{thm:free} shows the objective cannot determine it.}
\label{fig:polar}
\end{figure}
}

\title{\bfseries Whitening Inverts the Hierarchy:\\
What the Norm of a Whitened Embedding Measures}
\author{
    Mohammed AHNOUCH
    \\
    Université Paris 1
    \\
    Paris, France
    \and
    Lotfi ELAACHAK
    \\
    Faculty of Science and Technology of Tangier
    \\
    Abdelmalek Essaadi University
    \\
    Tangier, Morocco
}

\date{}

\begin{document}

\maketitle

\begin{abstract}
\noindent
Whitening a foundation-model embedding and reading off its squared norm has become a training-free
likelihood surrogate, justified by a single observation: the whitened coordinates are close to
standard normal. We show that this observation is guaranteed by the projection central limit
theorem, so it says nothing about the joint law, and that the joint law is not Gaussian. Against
clones with identical mean and covariance the whitened radius is over-dispersed, for every
encoder we tested and across three training objectives, and the published agreement between
empirical and theoretical norm statistics is an algebraic identity that in-sample whitening imposes
on any distribution.Whitening divides every direction by its
eigenvalue and so inverts the encoder's hierarchy, moving the mass of the norm off the semantic
subspace and onto the near-degenerate directions the encoder treats as noise, where a single
input-dependent scale sets the variance; we identify that scale from two moments and confirm it by
predicting, with no free parameter, how strongly disjoint halves of the spectrum move together. The
squared whitened norm is therefore a Mahalanobis estimate of semantic atypicality rather than a
log-likelihood, and that identification explains both its successes and its failures: it ranks and
detects well, agreeing with a nonparametric density estimate and with encoders that share no
objective, training data or supervision, while its Gaussian thresholds are wrong by orders of
magnitude. We prove that a contrastive objective built on cosine similarity cannot determine the radial law, derive in closed
form the group size at which a coordinatewise test can detect a given kurtosis, and give the
one-dimensional recalibration that restores nominal tail probabilities.

\end{abstract}

\vspace{1em}

\section{Introduction}
\label{sec:intro}

Let $f$ be an encoder and $z = f(x) \in \R^d$ a raw, unnormalised embedding. Write $\mu$ and
$\Sigma$ for its mean and covariance over a corpus, and
\begin{equation}
  w \;=\; \Sigma^{-1/2}(z-\mu)
  \label{eq:whiten}
\end{equation}
for the whitened embedding. A recent programme \citep{levi25,betser25,betser26,une26} observes that
the coordinates of $w$, and more generally its low-dimensional projections, are close to standard
normal \citep{df84,sudakov,klartag07}, and draws a practical conclusion: since
$-\log\phi(w) = \tfrac12\norm{w}^2 + \tfrac{d}{2}\log 2\pi$, the squared whitened norm is a
training-free likelihood surrogate. It is now used for out-of-distribution scoring, caption
quality, and the detection of generated images \citep{clide} and video \citep{video26}. The
observation is correct and the score is useful. What has been missing is the joint law, and the
joint law turns out to be more informative than the Gaussian would have been: the whitened norm is a
density estimate, and knowing what it estimates explains both where it succeeds and where it fails.

Our route is to take the published evidence seriously enough to reproduce it. We embedded $\Npub$
MS-COCO val2017 image-caption pairs \citep{coco} with CLIP ViT-L/14 \citep{clip,openclip} and
reproduced ten of the twelve published summary statistics to within five percent, which fixes the
protocol precisely, down to the basis in which the coordinates were tested. From there the paper
establishes four things. The published agreement between empirical and theoretical norm statistics
is an identity that in-sample whitening imposes on any distribution; evaluated against the genuine
$\chi_d$ prediction the same table is off by $\overTrueIm\times$ and $\overTrueTx\times$, and the
published grand averages sit $\zadIm$ and $\zadTx$ standard deviations from the null of the
protocol itself (\Cref{sec:evidence}). Against clones with identical mean and covariance the joint
law is neither Gaussian nor spherical, at $\swWeakAny$ to $\swStrongAny$ standard deviations, for
all six encoders we tested (\Cref{sec:joint}). The reason is that whitening inverts the encoder's hierarchy, moving
the share of the norm carried by the $k$ leading directions from $\hsRawMin$-$\hsRawMax$ percent
down to $k/d$, or $\hsWhMin$-$\hsWhMax$ percent: the norm measures the near-degenerate tail, where one per-input
scale, identified from two moments, sets the variance (\Cref{sec:mechanism}). That scale is semantic
atypicality, $\rho = \knnDiag$ against a nonparametric density estimate and $\shareRad$ shared across
encoders, and the Mahalanobis out-of-distribution score is its tail term to three decimals
(\Cref{sec:what}). \Cref{sec:free} proves that a contrastive objective built on cosine similarity cannot determine
the radial law and gives the recalibration.

\ifshort\else
\blockPolarFig
\fi

That contrastive training shapes the directional distribution is well understood
\citep{wangisola,haochen,balestriero22}, as is the modality gap \citep{liang22}, and whitening as a
fix for anisotropy predates this literature in NLP \citep{bertflow,su21,ethayarajh}.
\Citet{levi25} introduce \emph{conformity}, the mean cosine similarity of an instance to the corpus,
as a per-instance typicality measure on the unwhitened embedding. Conformity is directional and our
per-input scale is radial, so we measure the relation rather than assume it: conformity tracks the
\emph{raw} norm closely ($\rho = \confRawIm$ on CLIP ViT-L/14) and runs \emph{against} the whitened
radius ($\confRadIm$), the same inversion \Cref{sec:what} reports. On the detection side,
Mahalanobis distance on deep features \citep{lee18} and $k$-nearest-neighbour distance
\citep{sun22} are both established out-of-distribution scores. What is new here is an account of why they rank alike and calibrate
differently, and a demonstration that the Mahalanobis score is carried entirely by the spectral
tail.

\section{Whitening and the published evidence}
\label{sec:evidence}

Given $N$ embeddings, $\hat\mu$ and $\hat\Sigma$ are the sample mean and covariance and
$w_i = \hat\Sigma^{-1/2}(z_i-\hat\mu)$. When $\hat\Sigma$ is estimated in sample, three identities
hold exactly,
\begin{equation}
  \frac{1}{N}\sum_i w_i = 0,
  \qquad
  \frac{1}{N}\sum_i w_i w_i^\top = \Id_d,
  \qquad
  \frac{1}{N}\sum_i \norm{w_i}^2 = d ,
  \label{eq:identities}
\end{equation}
and in-sample whitening also compresses the radius: a standard Gaussian sample of $N = \Npub$ points
in $d = \Dim$ dimensions has radial standard deviation $\refSd$ under this protocol rather than the
population value $\sdchi$ of $\chi_d$, because $d/N$ of the variance directions are fitted away.
Every comparison below is against a covariance-matched clone through the identical pipeline, so this
and every other pipeline effect cancels. In polar coordinates $w = R\,u$ with $R = \norm{w}$ and
$u \in \Sph^{d-1}$, a spherically symmetric law with radial density $p_R$ has
\begin{equation}
\begin{gathered}
  f(w) = \frac{p_R(r)}{S_{d-1}\,r^{\,d-1}},
  \qquad
  S_{d-1} = \frac{2\pi^{d/2}}{\Gamma(d/2)}, \\[2pt]
  -\log f(w) = -\log p_R(r) + (d-1)\log r + \log S_{d-1} ,
\end{gathered}
  \label{eq:sphericaldensity}
\end{equation}
and the Gaussian is the case $p_R = \chi_d$, for which this collapses to
$\tfrac12 r^2 + \tfrac d2\log2\pi$. Equivalently, $\Nn(0,\Id_d)$ places almost all its mass near
the sphere of radius $\sqrt d$: $\Ex\norm{W}\sim\sqrt d$ with fluctuations of constant order, so
the relative radial width is $O(d^{-1/2})$ \citep[Thm.~3.1.1 and Rem.~3.1.2]{vershynin}. That is
the thin spherical shell of a high-dimensional Gaussian; whether whitened embeddings concentrate
the same way is measured below. Two consequences of \cref{eq:identities} govern what the
published protocol could and could not see.

\begin{lemma}[in-sample whitening]
\label{prop:insample}
Let $w_i = R_i u_i$ be in-sample whitened, with $\hat s^2 = \frac1N\sum_i(R_i-\hat m)^2$. Then
(a) for any distribution $\hat s^2 = d - \hat m^2$; and (b)
$\frac1N\sum_i R_i^2\inner{u_i}{v}^2 = 1$ for every unit vector $v$. In particular the
radius-weighted directional energy in (b) is constant in $v$ by construction, so it carries no
information about a radius-direction coupling at any sample size.
\end{lemma}

\begin{proof}
Claim (a) is the third identity of \cref{eq:identities}, $\frac1N\sum_i R_i^2 = d$, combined with
$\hat s^2 = \frac1N\sum_i R_i^2 - \hat m^2$; claim (b) is the second identity,
$v^\top\big(\frac1N\sum_i w_iw_i^\top\big)v = 1$, written in polar form.
\end{proof}

\ifshort
\else
\Citet{betser25} report empirical and theoretical norm statistics side by side, from their Eq.~(8),
\blockEqEight
which are the correct $\chi_d$ formulas and give $\Ex[S] = \muTheo$ and
$\mathrm{Std}(S) = \sdTheo$ at $d = \Dim$. The standard deviation depends on $d$ alone, so any two
rows sharing $d$ must share it, and the published rows do not: the theoretical standard deviation is
$\pubSdTheoIm$ for images and $\pubSdTheoTx$ for captions, while the theoretical mean is
$\pubMeanTheo$ in both rows, correctly depending on $d$ alone. One formula of \cref{eq:eight} was
evaluated theoretically and the other was not. Substituting the empirical mean $\hat m$ for
$\mu_S$ in the second reproduces both entries and both reported deviations:

\blockSubstTab

By \Cref{prop:insample}(a), $\sqrt{d-\hat m^2}$ is the sample standard deviation recomputed from
the sample mean, so its agreement with the sample standard deviation is arithmetic and would hold for
uniform vectors, for a scale mixture, for anything. The caption row shows the substitution a second
way: there $d - \hat m^2 = \subRootTx$ is negative, so the published entry can only be
$\sqrt{|d-\hat m^2|}$, and \cref{eq:eight} read literally returns an imaginary standard deviation.
Against the genuine prediction the same table reads very differently. The $\chi_{\Dim}$ standard
deviation is $\sdTheo$; the reported values $\pubSdIm$ and $\pubSdTx$ exceed it by $\overTrueIm$ and
$\overTrueTx$ times. The formula is right and the evaluation is not, and the lesson is the protocol
adopted throughout this paper: a distributional claim needs a matched null.
\fi

\begin{table}[htbp]
\centering\small
\caption{Published summary statistics against the same statistics on our own embeddings, in-sample
whitening at the published group size $\gpub$. ``Gaussian'' is an exact Gaussian of the same shape
through the same pipeline, its radial sd the mean of $\refSeeds$ draws, spread $\pm 0.006$.}
\label{tab:replicate}
\setlength{\tabcolsep}{4.5pt}
\begin{tabular}{lccc ccc}
\toprule
 & \multicolumn{3}{c}{images} & \multicolumn{3}{c}{captions} \\
\cmidrule(lr){2-4}\cmidrule(lr){5-7}
 & published & measured & Gaussian & published & measured & Gaussian \\
\midrule
radial mean & $\pubMeanIm$ & $\mRadIm$ & $\muTheo$ & $\pubMeanTx$ & $\mRadTx$ & $\muTheo$ \\
radial sd   & $\pubSdIm$   & $\mSdIm$  & $\refSd$  & $\pubSdTx$   & $\mSdTx$  & $\refSd$ \\
AD pass     & $\adpassIm\%$ & $\mGrpadpassIm\%$ & $\mGrpadpassGau\%$ & $\adpassTx\%$ & $\mGrpadpassTx\%$ & $\mGrpadpassGau\%$ \\
mean $\Ad$  & $\adavgPubIm$ & $\mGrpadavgIm$ & $\mGrpadavgGau$ & $\adavgPubTx$ & $\mGrpadavgTx$ & $\mGrpadavgGau$ \\
DP pass     & $\dppassIm\%$ & $\mGrpdppassIm\%$ & $100.0\%$ & $\dppassTx\%$ & $\mGrpdppassTx\%$ & $100.0\%$ \\
mean $p$    & $\dpavgPubIm$ & $\mGrpdpavgIm$ & $\mGrpdpavgGau$ & $\dpavgPubTx$ & $\mGrpdpavgTx$ & $\mGrpdpavgGau$ \\
\bottomrule
\end{tabular}
\end{table}

The protocol splits the $d = \Dim$ whitened coordinates of $\Npub$ embeddings into $20$ groups of
$\gpub$ and submits each coordinate of each group to an Anderson-Darling \citep{ad54} and a
D'Agostino-Pearson \citep{dagostino73} test at the estimated-parameter critical value
$\Ad_{\mathrm{crit}} = 0.752$ \citep[Table~4.7]{stephens86}. Run on our embeddings
(\Cref{tab:replicate}), every image statistic reproduces to within $\repImMax\%$ and all three
caption test statistics to within $\repTxTestMax\%$, which fixes the protocol down to the basis:
principal axes give a mean Anderson-Darling statistic of $\mGrpadavgIm$ against the published
$\adavgPubIm$, the symmetric square root $\hat\Sigma^{-1/2}$ gives $\adZca$. The caption radial law is the one row that does not
reproduce, low by $\repTxRadMean\%$ on the mean and $\repTxRadSd\%$ on the standard
deviation, and that row fails the identity that fixes the rest. In-sample whitening forces
$\hat m^2 + \hat s^2 = d$; ours gives $\mIdentTx$ against $\Dim$ ($\mIdentPctTx\%$, the compression
noted above), the published row $844.4$ ($+9.95\%$), which no in-sample whitening of a single sample
can produce. It is the same row that needs a negative number under the square root.

\ifshort\else
\blockZTab
\fi

Averaging over $\Dim$ coordinates and $20$ groups makes the sampling error of the protocol's grand
averages very small, and \Cref{tab:z}{} shows what the published
averages say once that null is in hand: the published statistics reject normality at $\zadIm$ and $\zadTx$ standard deviations. The
pass rates are outside the null too, degenerately, since an exact Gaussian passes $100.0\%$ of
coordinates on both tests in all $\nullreps$ replications. How a $98\%$ pass rate coexists with a
$62\sigma$ rejection is a matter of group size, and the answer is exact.

\begin{proposition}[power of a coordinatewise test]
\label{prop:power}
Let $X = \sqrt{S}\,G$ with $G \sim \Nn(0,1)$ independent of $S > 0$, $\Ex S = 1$, $\Var S = V$, so
that $X$ has unit variance and excess kurtosis $\gamma_2 = 3V$. As $V \to 0$ with
$\Ex|S-1|^3 = o(V)$, the Anderson-Darling discrepancy of $X$ from the normal,
$\Delta^2 = \int (F-\Phi)^2/(\Phi(1-\Phi))\dx\Phi$, satisfies $\Delta^2 = \kappa V^2 + o(V^2)$ with
\begin{equation}
  \kappa \;=\; \frac{1}{64}\int_{\R} \frac{x^2(3-x^2)^2\,\phi(x)^3}{\Phi(x)\,(1-\Phi(x))}\dx x
  \;=\; \kappaADfive .
  \label{eq:kappa}
\end{equation}
\end{proposition}

The proof is in \Cref{app:proofs}.

To turn $\kappa$ into a group size we use the growth of the mean statistic under a fixed
alternative. Since $\Ad_n/n = \int (F_n - F_0)^2 / [F_0(1-F_0)] \dx F_0$ and $F_n \to F$ uniformly
almost surely, Glivenko-Cantelli gives $\Ad_n/n \to \Delta^2$ almost surely whenever $\Delta^2$ is
finite, which in the perturbative regime is part of \Cref{prop:power}; so
$\Ex[\Ad_n] \simeq \mu_0 + n\Delta^2$, with $\mu_0$ the null mean and the asymptotic null theory of
the statistic that of \citet{ad52}. The critical value $\Ad_{\mathrm{crit}} = 0.752$ is the
estimated-parameter one \citep[Table~4.7]{stephens86}. We take $\mu_0 = \muzero$, simulated for this protocol rather than
read from a table, because the protocol averages over groups. Solving for $n$ at $V = \gamma_2/3$,
\begin{equation}
  n^\star(\gamma_2) \;\approx\; \frac{9\,(\Ad_{\mathrm{crit}} - \mu_0)}{\kappa\,\gamma_2^{2}}
  \;=\; \frac{\nstarConst}{\gamma_2^{2}} .
  \label{eq:nstar}
\end{equation}
\ifshort
\else
\Cref{eq:nstar} composes two expansions taken in different limits ($n \to \infty$ at fixed
$\Delta^2$, and $V \to 0$ at fixed $n$), so it is calibrated rather than derived, and we measure its
accuracy. A coordinate's mean statistic stays below the critical value at group size
$n$ when its kurtosis is below $3\sqrt{(\Ad_{\mathrm{crit}}-\mu_0)/(\kappa n)}$, which is $\pthrb$
at $n = \gpub$ and $\pthrf$ at $n = \Npub$. The whitened image coordinates have per-coordinate kurtosis of mean
$\kmeanIm$ and root mean square $\krmsIm$ in principal axes, the captions $\kmeanTx$ and $\krmsTx$,
so the group size at which the average coordinate fails is $n^\star = \nstarMIm$ and $\nstarMTx$,
$\nstarFacIm$ and $\nstarFacTx$ times the $\gpub$ used. Applying the threshold to the measured kurtosis of each coordinate
predicts the pass rate at every group size:

\blockPassTab

The leading-order formula under-predicts the pass rate below $n = 1000$, where the truncation of
the Anderson-Darling weight at the resolution of the extreme order statistics lowers the effective
discrepancy, and from $n = 2500$ it is accurate to within three and a half points. At $n = \Npub$
the published evidence collapses: tested in one group rather than twenty, the same $\Npub$
embeddings pass at $\mAlladpassIm\%$ and $\mAlladpassTx\%$ while the Gaussian control passes at
$\mAlladpassGau\%$.
\fi
The published protocol was therefore run at a group size where \cref{eq:nstar} predicts a pass,
and its grand averages, which carry the information the pass rates discard, reject.

\section{The joint law: neither Gaussian nor spherical}
\label{sec:joint}

\begin{lemma}[polar characterisation; classical]
\label{lem:polar}
For $W \in \R^d$ with $W \neq 0$ a.s., $R = \norm{W}$, $U = W/R$: (i) $W$ is spherically symmetric
iff $U \sim \unif(\Sph^{d-1})$ and $U \perp R$; (ii) $W \sim \Nn(0,\Id_d)$ iff additionally
$R \sim \chi_d$. \citep[see][\S2.1]{fangkotzng}
\end{lemma}

The lemma has two hypotheses, and each can fail on its own. A \emph{scale mixture} $W = S\,G$, with
$S>0$ independent of $G\sim\Nn(0,\Id_d)$, keeps $U$ uniform but breaks $U \perp R$'s consequence
for the radial law: it is spherical and not Gaussian, so $\norm{w}^2$ still ranks by likelihood and
only the calibration is wrong. A \emph{sign-type law} $W_i = \varepsilon_i|V_i|$, with $\varepsilon$
Rademacher independent of an equicorrelated $V$, has exactly $\Nn(0,1)$ coordinates, correlated
squares, an over-dispersed radius, and is not spherical, and there $\norm{w}^2$ is not a likelihood
ranking at all. By \Cref{lem:polar} these are the two ways out, one per hypothesis, and there is no third: a
rotation-invariant law on $\R^2$ with independent coordinates is Gaussian by Maxwell's theorem
\citep[Thm.~0.0.1]{bryc95}, applied here to each coordinate pair.

Neither is visible to a coordinatewise test, but for different reasons, and the difference matters.
The sign-type law is invisible \emph{exactly}: its marginals are standard normal, so no test of any
power rejects them. The scale mixture is invisible only \emph{to the protocol as run}: its
coordinates are not normal, and by \Cref{prop:power} their excess kurtosis $3V$ becomes detectable
at group size $n^\star(3V)$, which at the measured spreads exceeds the $\gpub$ used. Coordinatewise
normality is thus evidence against neither.

Deciding between them needs a control matching the data's covariance, since a spiked spectrum
estimated from a finite reference half is noisier than an isotropic one. We build two clones of each
dataset: a Gaussian clone with the same mean and covariance, and a spherical clone with the same
mean and covariance carrying the data's own whitened radial law on uniform directions. Data and
clones go through an identical split-half pipeline with principal axes from the reference half, so
nothing is fitted on the data being scored, which \Cref{prop:insample}(b) makes necessary: in-sample
whitening puts the coupling at $0.02$ to $0.06$ where split-half whitening of the same embeddings
gives $\sCorrDatIm$ and $\sCorrDatTx$.

\begin{table}[htbp]
\centering\small
\caption{Data against covariance-matched clones; split-half whitening, axes from the reference half,
five splits. A spherical law has the same coordinate kurtosis in every basis and a radius
independent of direction. Both clones do; the data does neither.}
\label{tab:clone}
\setlength{\tabcolsep}{4pt}
\begin{tabular}{llccc}
\toprule
 & statistic & data & Gaussian clone & spherical clone \\
\midrule
\multirow{4}{*}{images}
 & radial sd                    & $\sRsdDatIm$  & $\sRsdGclIm$  & $\sRsdSclIm$ \\
 & kurtosis, principal axes     & $\sKpcaDatIm$ & $\sKpcaGclIm$ & $\sKpcaSclIm$ \\
 & kurtosis, encoder's own axes & $\sKrawDatIm$ & $\sKrawGclIm$ & $\sKrawSclIm$ \\
 & $\corr(R, |\inner{u}{v_1}|)$ & $\sCorrDatIm$ & $\sCorrGclIm$ & $\sCorrSclIm$ \\
\midrule
\multirow{4}{*}{captions}
 & radial sd                    & $\sRsdDatTx$  & $\sRsdGclTx$  & $\sRsdSclTx$ \\
 & kurtosis, principal axes     & $\sKpcaDatTx$ & $\sKpcaGclTx$ & $\sKpcaSclTx$ \\
 & kurtosis, encoder's own axes & $\sKrawDatTx$ & $\sKrawGclTx$ & $\sKrawSclTx$ \\
 & $\corr(R, |\inner{u}{v_1}|)$ & $\sCorrDatTx$ & $\sCorrGclTx$ & $\sCorrSclTx$ \\
\bottomrule
\end{tabular}
\end{table}

Two properties in \Cref{tab:clone} are impossible for a spherical law. The kurtosis is
basis-dependent: both clones show the same coordinate kurtosis in the encoder's own axes as in the
principal axes, because a spherical excess is isotropic, whereas the data has coordinates
indistinguishable from normal in the encoder's basis and strongly leptokurtic in principal axes, $\szKrawSIm\,\sigma$
and $\szKrawSTx\,\sigma$ from the spherical clone. And the radius depends on the direction: under
sphericity $R \perp U$, yet alignment with the reference half's leading principal axis predicts the
radius at $\sCorrDatIm$ for images against $\sCorrGclIm$ for the Gaussian clone, $\szCorrGIm$ and
$\szCorrSIm$ standard deviations away.

\begin{table}[htbp]
\centering\small
\caption{Six encoders, the same $\Npub$ images, the same battery. ``raw cv'' is the coefficient of
variation of $\norm{z}$ before whitening; ``over'' is the whitened radial sd divided by a Gaussian
clone's under the split-half protocol of \Cref{tab:clone}, which is why CLIP-L/14 reads
$\ClILover\times$ here and $\mOverIm\times$ in \Cref{tab:replicate}.}
\label{tab:sweep}
\setlength{\tabcolsep}{3.4pt}
\begin{tabular}{lcl ccccc}
\toprule
encoder & $d$ & ends in & raw cv & over & $\gamma_2$ raw & $\gamma_2$ PC & $\corr(R,|u_1|)$ \\
\midrule
CLIP ViT-L/14, openai   & \ClILd  & linear projection & \ClILcv  & \ClILover$\times$  & \ClILkraw  & \ClILkpca  & \ClILcorr \\
CLIP ViT-B/32, openai   & \ClIBod & linear projection & \ClIBocv & \ClIBoover$\times$ & \ClIBokraw & \ClIBokpca & \ClIBocorr \\
CLIP ViT-B/32, LAION-2B & \ClIBld & linear projection & \ClIBlcv & \ClIBlover$\times$ & \ClIBlkraw & \ClIBlkpca & \ClIBlcorr \\
CLIP ViT-B/32, DataComp & \ClIBdd & linear projection & \ClIBdcv & \ClIBdover$\times$ & \ClIBdkraw & \ClIBdkpca & \ClIBdcorr \\
SigLIP ViT-B/16, WebLI  & \ClISid & attention pool    & \ClISicv & \ClISiover$\times$ & \ClISikraw & \ClISikpca & \ClISicorr \\
DINOv2 ViT-B/14         & \ClIDid & LayerNorm         & \ClIDicv & \ClIDiover$\times$ & \ClIDikraw & \ClIDikpca & \ClIDicorr \\
\bottomrule
\end{tabular}
\end{table}

\Cref{tab:sweep} repeats the battery on six encoders. Every row gives the same verdict, the weakest
departure across the six standing at $\swWeakKraw\,\sigma$ on kurtosis anisotropy,
$\swWeakCorr\,\sigma$ on the coupling and $\swWeakRsd\,\sigma$ on over-dispersion. This is not a property of CLIP, of InfoNCE, or of image-text training: it
holds for a sigmoid pairwise loss and for a self-distillation objective with no text at all. The
over-dispersion spans only a factor of $\overBand$ across the six ($\overMin\times$ to
$\overMax\times$) while the pre-whitening norm dispersion spans $\cvBand$, and the two are
uncorrelated ($\rho = \swSpear$, $p = \swSpearP$), a dissociation the next two sections explain.

\section{The mechanism: whitening inverts the hierarchy}
\label{sec:mechanism}

Write the embedding as a low-rank semantic head plus a residual whose scale depends on the input,
\begin{equation}
  z \;=\; A\,c(x) \;+\; \sigma(x)\,\varepsilon ,
  \label{eq:model}
\end{equation}
with $A \in \R^{d\times k}$ carrying the leading eigenvalues, $k$ the effective rank, and
$\varepsilon$ isotropic. Measured effective ranks are $\MechImeff$ of $\Dim$ for CLIP ViT-L/14 and
$\MechDeff$ for DINOv2, so $k \ll d$ throughout. The encoder's leading $k$ directions carry
$\hsRawMin$ to $\hsRawMax\%$ of the raw norm; whitening divides every direction by its eigenvalue, and
what that does to the norm is exact.

\begin{proposition}[whitening inverts the hierarchy]
\label{prop:inversion}
Let $\Sigma = \sum_j \lambda_j v_jv_j^\top$ with $\lambda_1 \ge \dots \ge \lambda_d$ and let $P_k$
project onto $\mathrm{span}(v_1,\dots,v_k)$. Then
\begin{equation}
  \frac{\Ex\norm{P_k(z-\mu)}^2}{\Ex\norm{z-\mu}^2} = \frac{\sum_{j\le k}\lambda_j}{\sum_j\lambda_j},
  \qquad
  \frac{\Ex\norm{P_k w}^2}{\Ex\norm{w}^2} = \frac{k}{d},
  \label{eq:inversion}
\end{equation}
and both hold as exact sample identities under in-sample whitening. Suppose further that the tail is
conditionally \emph{centred} Gaussian given the input,
$P_k^\perp w \mid x \sim \Nn(0, D_x)$, with $\Ex D_x = \Id_m$ and $m = d-k$. Then the tail energy $T = \norm{P_k^\perp w}^2$ has
\begin{equation}
\begin{gathered}
  \Ex T = m, \qquad
  \Var T = \Var(\tr D_x) + 2\,\Ex\,\tr(D_x^2), \\
  \Cov(T_1,T_2) = \Cov\big(\tr D_x^{(1)}, \tr D_x^{(2)}\big) + 2\,\Ex\,\tr\big(D_x^{(12)}D_x^{(21)}\big)
\end{gathered}
  \label{eq:moments}
\end{equation}
for the energies $T_1, T_2$ of any two disjoint blocks of tail directions, where $D_x^{(a)}$ denotes
the diagonal block of $D_x$ on block $a$ and $D_x^{(12)}, D_x^{(21)}$ its off-diagonal blocks; each
tail coordinate $j$ has excess kurtosis $3\Var\big((D_x)_{jj}\big)$. In particular, if $D_x = s(x)\Id_m$ with
$\Ex s = 1$ and $\Var s = V$, then relative to a Gaussian clone, and for two halves of the tail,
\begin{equation}
  \frac{\Var T}{\Var_{\Nn} T} = 1 + V\Big(1 + \frac m2\Big),
  \qquad
  \corr(T_1,T_2) = \frac{V}{V + \tfrac4m(1+V)},
  \qquad
  \gamma_2(w_j) = 3V ,
  \label{eq:onescale}
\end{equation}
while if the two halves carry scales $s_1, s_2$ with $\corr(s_1,s_2) = \rho_s$ the second
expression is multiplied by $\rho_s$.
\end{proposition}

The proof is in \Cref{app:proofs}. The first display is the inversion: for CLIP ViT-L/14 the head
share of the norm falls from $\MechImhsraw\%$ to $\MechImhswh\%$, and the $d-k$ directions the
encoder treats as noise, amplified to unit variance each, come to dominate the norm by count.
The second display says what the norm then measures, and the third identifies it. Because
$m \approx \Dim - \MechImk$, the factor $m/4$ multiplies any per-input variation of the residual
scale by more than a hundred, which is why a scale with a coefficient of variation of a third
produces a radius over-dispersed sixfold while the coordinates, whose excess kurtosis is only $3V$,
stay close to normal. The first two displayed moments are two equations in the two unknowns
$(V,\rho_s)$, so the measured over-dispersion and the measured halves correlation determine both,
with no free parameter left over. Solving them jointly matters: reading $V$ off the over-dispersion
alone, as the single-scale formula invites, biases it low by the factor $(1+\rho_s)/2$, which is
negligible at $\rho_s$ near one but a factor $\MechBvbias$ for CLIP ViT-B/32. On units: \cref{eq:onescale} is a ratio of variances of the tail energy
and \Cref{tab:sweep} a ratio of radial standard deviations, which agree to first order by the
delta method, $\mathrm{sd}(R) \approx \mathrm{sd}(T)/(2\sqrt{\Ex T})$.
\Cref{tab:mechanism} reports the fit for five encoder-modality pairs.

\begin{table}[htbp]
\centering\small
\caption{The inversion and the scale model of \Cref{prop:inversion}, in-sample PCA whitening.
``head share'' is the share of the norm carried by the top-$k$ directions before and after
whitening. ``over'' is the tail-energy variance relative to a Gaussian clone and ``halves'' the
correlation between the energies of the two halves of the tail; these two measurements determine
$(V,\rho_s)$ jointly through \cref{eq:onescale}. ``$V$, one scale'' is what the over-dispersion
alone would give under $\rho_s = 1$, shown to expose the bias. ``tail $\gamma_2$'' is the tail
coordinate kurtosis, as implied by the fitted scale ($3V$) and as measured; the excess above $3V$ is
scale variation within a half, which $\rho_s$ does not capture. ``tail var'' is the share of
$\Var(\norm{w}^2)$ carried by the tail.}
\label{tab:mechanism}
\footnotesize
\setlength{\tabcolsep}{3pt}
\begin{tabular}{lccccccccccc}
\toprule
 & $k$ & \multicolumn{2}{c}{head share} & over & halves & $V$ & $\rho_s$ & $V$, one
 & \multicolumn{2}{c}{tail $\gamma_2$} & tail \\
\cmidrule(lr){3-4}\cmidrule(lr){10-11}
 & & raw & whitened & & & & & scale & $3V$ & measured & var \\
\midrule
CLIP-L/14 images   & \MechImk & \MechImhsraw\% & \MechImhswh\% & \MechImtvr & \MechImhalves & \MechImV & \MechImrho & \MechImVone & \MechImktp & \MechImkt & \MechImtailvar\% \\
CLIP-L/14 captions & \MechTxk & \MechTxhsraw\% & \MechTxhswh\% & \MechTxtvr & \MechTxhalves & \MechTxV & \MechTxrho & \MechTxVone & \MechTxktp & \MechTxkt & \MechTxtailvar\% \\
CLIP-B/32 images   & \MechBk & \MechBhsraw\% & \MechBhswh\% & \MechBtvr & \MechBhalves & \MechBV & \MechBrho & \MechBVone & \MechBktp & \MechBkt & \MechBtailvar\% \\
SigLIP images      & \MechSk & \MechShsraw\% & \MechShswh\% & \MechStvr & \MechShalves & \MechSV & \MechSrho & \MechSVone & \MechSktp & \MechSkt & \MechStailvar\% \\
DINOv2 images      & \MechDk & \MechDhsraw\% & \MechDhswh\% & \MechDtvr & \MechDhalves & \MechDV & \MechDrho & \MechDVone & \MechDktp & \MechDkt & \MechDtailvar\% \\
\midrule
Gaussian clone     & \MechGk & n/a & n/a & $1$ & \MechGhalves & $0$ & n/a & $0$ & $0$ & \MechGkt & \MechGtailvar\% \\
\bottomrule
\end{tabular}
\end{table}

The fit is coherent in every row, and the two parameters separate cleanly. The identified $V$ runs
$\VMin$ to $\VMax$, a coefficient of variation of the per-input scale between a third and a half,
and $\rho_s$ runs $\rhoMin$ to $\rhoMax$: in every encoder most of the scale is common to the whole
tail, and what is left is idiosyncratic to a half. The two appear separately in the coordinate
kurtosis, which exceeds the $3V$ a common scale implies by the idiosyncratic part, so that across
the five rows the ordering of the excess is the ordering of $1-\rho_s$.
\ifshort\else
At $\rho_s$ near one the
measured tail kurtosis is close to $3V$ (DINOv2, $\rho_s = \MechDrho$, kurtosis $\MechDkt$ against
$\MechDktp$), and at $\rho_s$ small it is well above (CLIP ViT-B/32, $\rho_s = \MechBrho$,
kurtosis $\MechBkt$ against $\MechBktp$).
\fi
The tail carries $\tailvarMin$-$\tailvarMax\%$ of $\Var(\norm{w}^2)$, the head
contributing about $2k$ as a Gaussian would, except for DINOv2, whose leading decile is itself
heavy-tailed. \Cref{fig:mechanism} shows the three facts that matter most: the non-Gaussianity lives
where the eigenvalues are smallest and is absent from the clone; two disjoint halves of the tail move
together at $\MechImhalves$ against $\MechGhalves$ for the clone; and the scale that moves them is
semantic atypicality, the subject of the next section. Coordinates all but normal in one privileged
basis, a radius over-dispersed sixfold, and dependence living between the squares is the sign-type
law of \Cref{sec:joint}, introduced there only to show that coordinatewise evidence cannot exclude
non-sphericity; \cref{eq:model} says why encoders produce it.

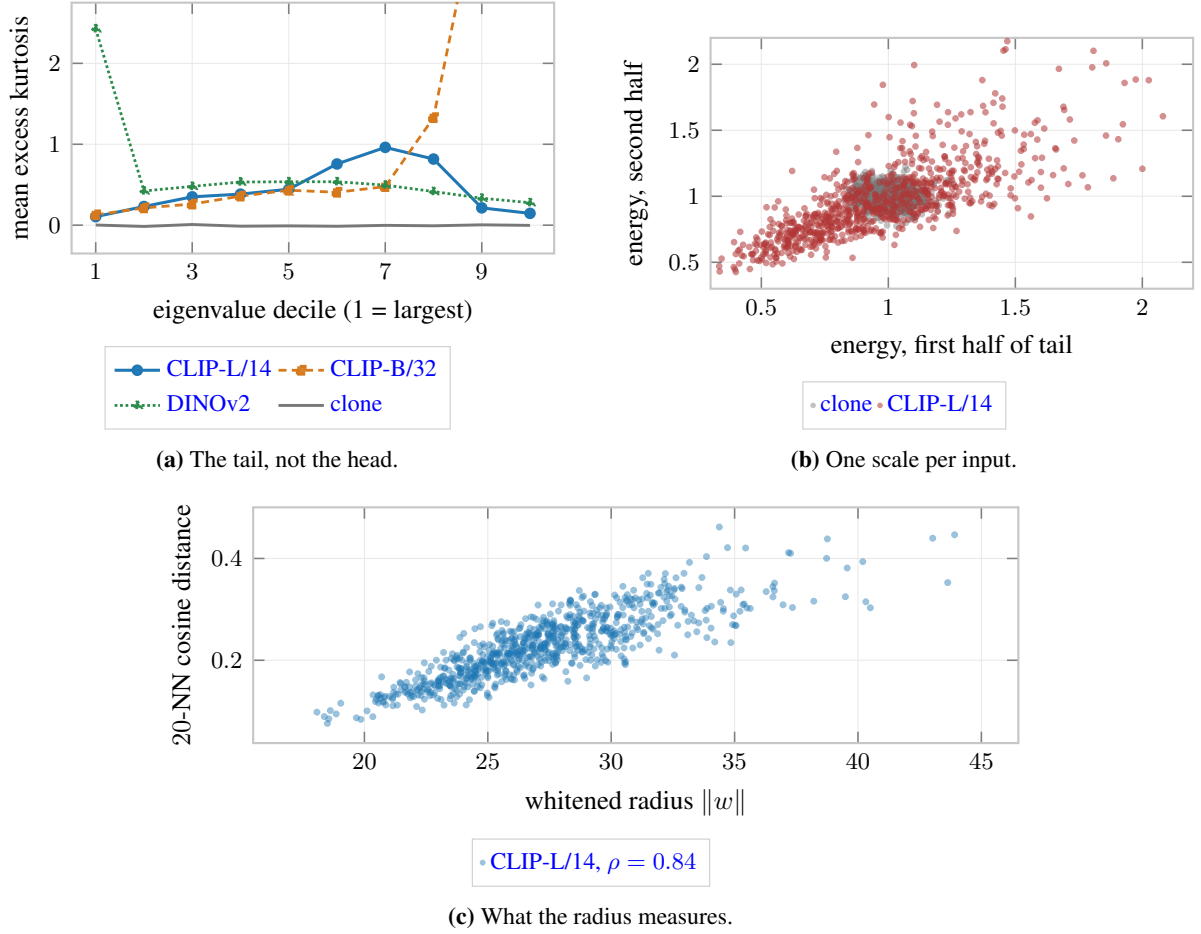
\begin{figure}[t]
\centering
\ifarxiv
\begin{subfigure}{0.49\textwidth}
\centering
\begin{tikzpicture}
\begin{axis}[width=\textwidth, height=4.9cm,
   xlabel={eigenvalue decile (1 = largest)}, ylabel={mean excess kurtosis},
   xmin=0.5, xmax=10.5, ymin=-0.35, ymax=2.75, xtick={1,3,5,7,9},
   legend columns=2, legend to name=figMechA,
   legend cell align=left,
   /tikz/every even column/.append style={column sep=5pt}]
 \addplot[frClipL, line width=1.15pt, mark=*, mark size=1.6pt,
          mark options={fill=frClipL}] table[x=decile,y=clipLimg] {decile.dat};
 \addlegendentry{CLIP-L/14}
 \addplot[frClipB, densely dashed, line width=1.05pt, mark=square*, mark size=1.5pt,
          mark options={fill=frClipB}] table[x=decile,y=clipB] {decile.dat};
 \addlegendentry{CLIP-B/32}
 \addplot[frDino, densely dotted, line width=1.05pt, mark=triangle*, mark size=1.7pt,
          mark options={fill=frDino}] table[x=decile,y=dinov2] {decile.dat};
 \addlegendentry{DINOv2}
 \addplot[frClone, line width=1.1pt] table[x=decile,y=clone] {decile.dat};
 \addlegendentry{clone}
\end{axis}
\end{tikzpicture}\\[0.35em]
\ref{figMechA}
\caption{The tail, not the head.}
\end{subfigure}\hfill
\begin{subfigure}{0.49\textwidth}
\centering
\begin{tikzpicture}
\begin{axis}[width=\textwidth, height=4.9cm,
   xlabel={energy, first half of tail}, ylabel={energy, second half},
   xmin=0.3, xmax=2.2, ymin=0.3, ymax=2.2,
   legend columns=2, legend to name=figMechB,
   legend cell align=left,
   /tikz/every even column/.append style={column sep=5pt}]
 \addplot[only marks, mark=*, mark size=0.85pt, frClone, opacity=0.45]
   table[x=g1,y=g2] {tailhalves.dat};
 \addlegendentry{clone}
 \addplot[only marks, mark=*, mark size=0.85pt, frData, opacity=0.55]
   table[x=t1,y=t2] {tailhalves.dat};
 \addlegendentry{CLIP-L/14}
\end{axis}
\end{tikzpicture}\\[0.35em]
\ref{figMechB}
\caption{One scale per input.}
\end{subfigure}

\vspace{0.55em}
\begin{subfigure}{0.72\textwidth}
\centering
\begin{tikzpicture}
\begin{axis}[width=\textwidth, height=4.7cm,
   xlabel={whitened radius $\norm{w}$}, ylabel={$20$-NN cosine distance},
   legend to name=figMechC, legend cell align=left]
 \addplot[only marks, mark=*, mark size=0.9pt, frClipL, opacity=0.45]
   table[x=radius,y=knn] {atypicality.dat};
 \addlegendentry{CLIP-L/14, $\rho = \knnClipL$}
\end{axis}
\end{tikzpicture}\\[0.35em]
\ref{figMechC}
\caption{What the radius measures.}
\end{subfigure}
\else
\begin{subfigure}{0.325\textwidth}
\centering
\begin{tikzpicture}
\begin{axis}[width=\textwidth, height=3.5cm,
   xlabel={eigenvalue decile (1 = largest)}, ylabel={mean excess kurtosis},
   xmin=0.5, xmax=10.5, ymin=-0.3, ymax=2.7, xtick={1,3,5,7,9},
   legend style={at={(0.5,0)}, anchor=north, yshift=-0.95cm, legend columns=2, font=\tiny,
                 /tikz/every even column/.append style={column sep=2pt}, inner xsep=1.5pt},
   legend cell align=left]
 \addplot[black, line width=1.0pt, mark=*, mark size=1.1pt] table[x=decile,y=clipLimg] {decile.dat};
 \addlegendentry{CLIP-L/14}
 \addplot[black, densely dashed, line width=0.9pt, mark=square*, mark size=1.1pt] table[x=decile,y=clipB] {decile.dat};
 \addlegendentry{CLIP-B/32}
 \addplot[black, dotted, line width=0.9pt, mark=triangle*, mark size=1.3pt] table[x=decile,y=dinov2] {decile.dat};
 \addlegendentry{DINOv2}
 \addplot[gray!80, line width=1.0pt] table[x=decile,y=clone] {decile.dat};
 \addlegendentry{clone}
\end{axis}
\end{tikzpicture}
\caption{The tail, not the head.}
\end{subfigure}\hfill
\begin{subfigure}{0.325\textwidth}
\centering
\begin{tikzpicture}
\begin{axis}[width=\textwidth, height=3.5cm,
   xlabel={energy, first half of tail}, ylabel={energy, second half},
   xmin=0.3, xmax=2.2, ymin=0.3, ymax=2.2,
   legend style={at={(0.5,0)}, anchor=north, yshift=-0.95cm, legend columns=2, font=\tiny,
                 /tikz/every even column/.append style={column sep=2pt}, inner xsep=1.5pt},
   legend cell align=left]
 \addplot[only marks, mark=*, mark size=0.5pt, gray!55] table[x=g1,y=g2] {tailhalves.dat};
 \addlegendentry{clone}
 \addplot[only marks, mark=*, mark size=0.5pt, black] table[x=t1,y=t2] {tailhalves.dat};
 \addlegendentry{CLIP-L/14}
\end{axis}
\end{tikzpicture}
\caption{One scale per input.}
\end{subfigure}\hfill
\begin{subfigure}{0.325\textwidth}
\centering
\begin{tikzpicture}
\begin{axis}[width=\textwidth, height=3.5cm,
   xlabel={whitened radius $\norm{w}$}, ylabel={$20$-NN cosine distance},
   legend style={at={(0.5,0)}, anchor=north, yshift=-0.95cm, font=\tiny},
   legend cell align=left]
 \addplot[only marks, mark=*, mark size=0.5pt, black] table[x=radius,y=knn] {atypicality.dat};
 \addlegendentry{CLIP-L/14, $\rho = \knnClipL$}
\end{axis}
\end{tikzpicture}
\caption{What the radius measures.}
\end{subfigure}
\fi
\caption{The mechanism in three panels. (a) Whitening amplifies the encoder's least-informative
directions, and that is where the non-Gaussianity is; the Gaussian clone is flat at zero. (b) Two
disjoint halves of the tail have energies that move together for real embeddings and not for the
clone: one per-input scale multiplies the whole tail. (c) That scale is semantic atypicality.}
\label{fig:mechanism}
\end{figure}

\ifshort
\else
\Cref{eq:onescale} also says where the magnitude of the effect should vary: the factor $m/4$ grows
with the tail, so the lower the effective rank of an encoder's representation of a corpus, the more
of the norm one scale multiplies. \Citet{betser25} already report that the corpus used to estimate the whitening
transform matters, their cross-dataset table showing the average Anderson-Darling statistic rising
from $0.489$ to $0.641$ on images when a Flickr8k test set is whitened with a COCO transform. We
encoded six corpora, COCO scenes, CIFAR-100, Oxford pets, Flowers-102, DTD textures and SVHN digits,
at $\nBlock$ images per block in two disjoint blocks per corpus, so that the within-corpus spread
provides a measured null; Oxford pets admits one block rather than two, giving
$\nCorpBlocks$ blocks in all.

\ifshort\else
\blockCorpusTab
\fi

Between-corpus spread is $\cvRatioCl\times$ the within-corpus null for both encoders ($F = \cvFCl$
and $\cvFDi$ on $(5,5)$ degrees of freedom, $p = \cvPCl$), yet the two encoders order the corpora
differently ($\rho = \corpAcross$, $p = \corpAcrossP$): what matters is not the corpus but how
low-rank that encoder's representation of that corpus is. Across all $\nCorpBlocks$ encoder-corpus blocks,
over-dispersion is predicted by the effective rank ($\rho = \drvEff$, $p = \drvEffP$), by the
strength of the shared scale ($\drvTcor$, $p = \drvTcorP$) and by the dispersion of tail energy
($\drvCtail$, $p = \drvCtailP$).
\ifshort\else
Within CLIP alone, effective rank gives $\drvEffCl$ and tail dispersion $\drvCtailCl$.
\fi
Homogeneous corpora are the more over-dispersed, not the less, because a
tight cluster has low effective rank: SVHN under CLIP has effective rank $\svhnEff$ of $512$, so
$\svhnTail\%$ of the whitened norm is tail and one scale multiplies nearly all of it.
\fi

\section{What the norm measures}
\label{sec:what}

The mean cosine distance to an input's $20$ nearest neighbours is a nonparametric density estimate
that knows nothing about Gaussians. Against it, the whitened radius agrees at $\rho = \knnMin$ to
$\knnMax$ within an encoder, $\knnDiag$ on average, and the consensus radius over six encoders
agrees with the consensus distance at $\knnCons$. The raw norm runs the other way, $\rawknnMin$ to
$\rawknnMax$: prototypical inputs have large raw norms, the familiar ``norm as confidence'' effect,
and whitening inverts it. The quantity is a property of the input rather than of the encoder: six
encoders rank the same $\Npub$ images by whitened radius with mean cross-encoder agreement
$\shareRad$ ($\shareMin$-$\shareMax$), including $\shareMin$ between CLIP ViT-L/14 and DINOv2, which
share no objective, no training data and no text supervision, against $\shareRaw$ for the raw norm.
Whitening strips what is encoder-specific and leaves what is input-specific.

\begin{table}[htbp]
\centering\small
\caption{AUROC on a corpus-shift benchmark, averaged over five out-distributions, with the
in-distribution fitted on one block of $\nBlock$ images and evaluated on a disjoint block. ``tail''
and ``head'' are the parts of $\norm{w}^2$ beyond and within the effective rank. Mahalanobis$^{++}$
is Mahalanobis after $L^2$ normalisation.}
\label{tab:ood}
\begin{tabular}{lcccc}
\toprule
score & \multicolumn{2}{c}{CLIP ViT-B/32} & \multicolumn{2}{c}{DINOv2 ViT-B/14} \\
\cmidrule(lr){2-3}\cmidrule(lr){4-5}
 & COCO in & CIFAR-100 in & COCO in & CIFAR-100 in \\
\midrule
Mahalanobis, $\norm{w}^2$   & \oodMahCC  & \oodMahCF  & \oodMahDC  & \oodMahDF \\
\quad tail term only        & \oodTailCC & \oodTailCF & \oodTailDC & \oodTailDF \\
\quad head term only        & \oodHeadCC & \oodHeadCF & \oodHeadDC & \oodHeadDF \\
Mahalanobis$^{++}$          & \oodMppCC  & \oodMppCF  & \oodMppDC  & \oodMppDF \\
$k$-NN cosine distance      & \oodKnnCC  & \oodKnnCF  & \oodKnnDC  & \oodKnnDF \\
raw norm                    & \oodRawCC  & \oodRawCF  & \oodRawDC  & \oodRawDF \\
\bottomrule
\end{tabular}
\end{table}

\Cref{tab:ood} is the mechanism on a benchmark the detection literature recognises. The
Mahalanobis score is its tail term, agreeing to three decimals in all four settings, while the head
term is at chance or, for DINOv2, strongly anti-informative, out-of-distribution inputs projecting
less onto the in-distribution semantic directions. The whitened norm beats $k$-NN distance in three
of the four settings, the parametric tail estimate being sharper than the local one, and its
calibration on the same task is poor at the same time: at AUROC $\oodMahDC$, a threshold set at the
$1\%$ in-distribution tail flags only $\oodFlagDC\%$ of out-of-distribution inputs, and
$\oodFprDC\%$ of them pass at $95\%$ true-positive rate. $L^2$ normalisation before whitening,
the Mahalanobis$^{++}$ variant, does not act on the per-input scale: on the in-distribution test set
the normalised and unnormalised scores agree at $\rho = \oodSpRLDC$ and the tail-halves correlation
is unchanged ($\oodTcRawDC \to \oodTcLtDC$), because the shared scale lives in direction space,
which normalisation preserves. Where normalisation helps substantially (DINOv2 with CIFAR-100
in-distribution, $\oodMahDF \to \oodMppDF$) it is because the out-distribution's raw norms differ
from the in-distribution's and corrupt the covariance fit; where there is no such shift it slightly
hurts ($\oodMahCC \to \oodMppCC$). Its benefit is between distributions, not per input.

Two natural readings of the scale fail. The radius is uncorrelated with pixel-level
complexity ($\bppMin$ to $\bppMax$ against JPEG bytes per pixel, permutation standard deviation
$\statNull$), and against caption-based scene complexity it is positive but small ($\capMin$ to
$\capMax$, real at $\capSig$ standard deviations, under $2\%$ of variance). Atypicality here is
semantic and relational, isolation on the representation manifold, not a property of the pixels.

\section{The radial law is free, and the repair}
\label{sec:free}

\begin{theorem}[the radial law is free]
\label{thm:free}
Let $\mathcal{L}$ be an InfoNCE loss built on cosine similarity between encoders $f$ and $g$.
(a) For any measurable $\lambda:\mathcal{X}\to(0,\infty)$ and $\kappa:\mathcal{Y}\to(0,\infty)$,
$\mathcal{L}(\lambda f, \kappa g) = \mathcal{L}(f,g)$ exactly, for every batch and temperature.
(b) If the law of $x$ is atomless conditionally on $f(x)/\norm{f(x)}$, then for every probability
measure $\nu$ on $(0,\infty)$ there is a measurable $\lambda$ such that $\lambda f$ attains the same
loss as $f$ on every batch and has radial law $\nu$, independent of direction.
(c) Let $\rho$ be the law of the embedding with polar marginals $(\rho_U, \rho_{R|u})$, let
$\sigma$ be uniform on $\Sph^{d-1}$ and $\nu$ any law on $(0,\infty)$, write $\mathcal{L}(\rho)$
for the population loss, which by (a) depends on $\rho$ only through $\rho_U$, and suppose a
minimiser exists. Then for every $\beta>0$ every minimiser of
$\mathcal{L}(\rho) + \beta\KL(\rho\,\|\,\sigma\otimes\nu)$ has $R \sim \nu$ independent of $U$,
whatever $\rho_U$ is. The radial law is thus inherited from the reference measure rather than from
$\mathcal{L}$: $\nu = \chi_d$ makes the radius $\chi_d$, which is Gaussian exactly when the
directional factor is uniform, and any other $\nu$ gives a non-Gaussian law with the same
directional behaviour and the same value of the contrastive term.
\end{theorem}

\begin{proof}
(a) Each similarity depends on $f(x_i)$ only through $f(x_i)/\norm{f(x_i)}$, unchanged by
multiplication by $\lambda(x_i) > 0$, and the loss is a function of the similarity matrix alone.
(b) Take $\mathcal{X}$ standard Borel and $f \ne 0$ a.s., and condition on
$u = f(x)/\norm{f(x)}$. Fix a Borel isomorphism $g:\mathcal{X}\to[0,1]$ and let
$F_u(t) = \Prob(g(x) \le t \mid u)$ be a regular conditional distribution function, jointly
measurable in $(u,t)$ by construction. Put $T_u(x) = F_u(g(x))$: the conditional law is atomless, so
$T_u(x) \sim \unif(0,1)$ given every $u$, and $(u,x)\mapsto T_u(x)$ is measurable. Set
$\lambda(x) = F_\nu^{-1}(T_u(x))/\norm{f(x)}$. Then $\norm{\lambda(x)f(x)} \sim \nu$ conditionally
on every $u$, hence unconditionally and independently of $u$, and equality of losses is (a). (c) By (a), $\mathcal{L}(\rho)$ depends on $\rho$ only
through $\rho_U$, and $\KL(\rho\,\|\,\sigma\otimes\nu) = \KL(\rho_U\|\sigma) +
\Ex_{u\sim\rho_U}[\KL(\rho_{R|u}\|\nu)]$. The first term and $\mathcal{L}$ depend only on
$\rho_U$; the second is non-negative and vanishes iff $\rho_{R|u} = \nu$ for $\rho_U$-a.e.\ $u$.
\end{proof}

A derivation concluding that a contrastive objective induces a Gaussian has therefore put the
radial law into the reference measure. \Citet{betser26} prove Gaussianity of fixed-$k$ marginals, their
Corollary 1 giving $\sqrt d\,u_k \Rightarrow \Nn(0,\Id_k)$ for every fixed $k$, a Maxwell-Poincar\'e
statement with which our measurements agree: coordinate excess kurtosis in the encoders' own axes is
$\ClILkraw$ (\Cref{tab:sweep}). But $\norm{w}^2$ is a function of all $d$ coordinates, not of a
fixed-$k$ marginal, and marginal normality plus dependence does not give $\chi_d$; the dependence
of \Cref{fig:mechanism}(b) is what breaks it. Their Proposition 2, for unnormalised
encoders, reaches $\sqrt d\,z_k \Rightarrow \Nn(0, r_0^2\Id_k)$ under an assumption they name
thin-shell concentration, $r/r_0 \to 1$: the Gaussian shell property of \Cref{sec:evidence},
exactly the assumption \Cref{thm:free}(b) shows no
objective can supply, and exactly what our measurement finds violated, the whitened radius being
over-dispersed $\mOverIm\times$ rather than concentrating. Their Theorem 1 takes the regularised
route, with $\beta\KL(\rho\,\|\,\gamma^B_\lambda)$ against a Gaussian reference truncated to a
ball, $\gamma^B_\lambda \propto e^{-\lambda\norm{z}^2}\mathbf 1_B$; that is \Cref{thm:free}(c) with
$\nu$ the radial law of $\gamma^B_\lambda$, so the radial conclusion is the reference measure's.
Their framing is equivalent and worth making explicit, since
$\KL(\rho\,\|\,\gamma_\lambda)$ is up to constants a norm penalty minus an entropy: to regularise
towards low feature norm and high feature entropy \emph{is} to choose a Gaussian reference. Numerically,
on a random batch the loss is $\lossVal$; after multiplying every embedding by an independent
lognormal factor that changes the radial standard deviation from $\radBefore$ to $\radAfter$ it is
$\lossVal$ again, to every printed digit. Nor does the training recipe or the output layer set the radial
law empirically. Three ViT-B/32 encoders differing only in recipe have raw-norm dispersion differing
by $17$ standard errors and recall@1 differing by five points, yet statistically identical whitened
over-dispersion ($\chi^2 = \trioChi$ on two degrees of freedom, $p = \trioP$), and DINOv2, which
ends in a LayerNorm and pins its raw norm more than twice as tightly as any CLIP, is nonetheless
mid-range at $\ClIDiover\times$.
\ifshort\else
Recall@1 runs $\trioRlo\%$ to $\trioRhi\%$ and the three pool to $\trioPooled\times$; DINOv2's raw
norm has $\mathrm{cv} = \ClIDicv$ against $\cvMax$ for the widest CLIP.
\fi
\Cref{eq:model} explains the dissociation: whitening removes the
global scale and the covariance and leaves a quantity that depends on the shape of the residual, on
which a normalisation layer has no purchase.

\begin{proposition}[rankings and recalibration]
\label{prop:rank}
For a spherically symmetric law with radial density $p_R$: (a) if
$|\tfrac{\dx}{\dx r}\log p_R(r)| < (d-1)/r$ on the range of interest, then $-\log f(w)$ is strictly
increasing in $\norm{w}$, so $\norm{w}^2$ and the negative log-likelihood induce the same ordering
and every rank-based quantity is identical; (b) the only correction to the surrogate is a
reparameterisation of the radius,
\begin{equation}
  s(w) = -\log\hat p_R(\norm{w}) + (d-1)\log\norm{w} + \text{const},
  \label{eq:repair}
\end{equation}
with $\hat p_R$ a univariate density estimate; if $\hat p_R$ satisfies the inequality of (a) then
$s$ is increasing in $\norm{w}$ and every ranking is unchanged. For a twice-differentiable $p_R$ a
kernel estimate attains mean integrated squared error of order $n^{-4/5}$, a rate independent of
$d$ \citep[\S1.2]{tsybakov09}.
\end{proposition}

\begin{proof}
Differentiate \cref{eq:sphericaldensity}: $\tfrac{\dx}{\dx r}(-\log f) = -\tfrac{\dx}{\dx r}\log p_R
+ (d-1)/r > 0$ under (a), and rank statistics are invariant under strictly increasing
transformations. For (b), \cref{eq:sphericaldensity} shows that within the spherical family the
density is determined by $p_R$, so the surrogate is wrong exactly by the univariate factor
$p_R(r)/\chi_d(r)$, and a univariate kernel estimate attains the stated rate.
\end{proof}

The margin in (a) is a factor of $\monoMargin$ at the measured parameters, so within the spherical
family nothing can reorder the scores; the family is wrong, so the question is quantitative.
Conditioning the radial law on a single direction, the mildest available correction, moves the
ranking to Spearman $\rSpearIm$ (images) and $\rSpearTx$ (captions), with $\rKeepIm\%$ and
$\rKeepTx\%$ of the flagged top $1\%$ surviving: aggregate rank-based results are not in danger.
Thresholds are another matter.

\ifshort\else
\blockCalibTab
\fi

A nominal $1\%$ tail carries $\actIm\%$ of the image mass and $\actTx\%$ of the caption mass; at the
$99$th percentile of the image radius the $\chi_{\Dim}$ law is wrong by a factor $10^{\tailIm}$, and
for captions $10^{\tailTx}$.
\ifshort
\else
Cross-modal comparison inherits the same error, the surrogate reporting
gaps of $\crossGa$, $\crossGb$, $\crossGc$ nats at equal within-modality percentiles against
$\crossRa$, $\crossRb$, $\crossRc$ correct, and orientation contaminates the threshold, a single
nominal $1\%$ cut firing on $\rLoIm\%$ of image embeddings pointing away from the leading principal
axis and $\rHiIm\%$ of those aligned with it.
\fi
\Cref{tab:calib}{} shows \cref{eq:repair} restoring every
nominal level. \Cref{eq:repair} corrects the radial part, which is where the calibration error lives; the residual
departure from sphericity is low-dimensional, and conditioning on a few directional coordinates
handles it. Either way the argument of \Cref{prop:rank}(b) holds: no high-dimensional density
estimate is needed, which matters because estimating one in
$\R^{\Dim}$, even a log-concave one, is impractical at this sample size \citep{cule,kimsamworth}. The practical
recommendation is simpler still: calibrate on the empirical law of the tail energy, or use the
$k$-NN density the score approximates.

\section{Conclusion}
\label{sec:conclusion}

Whitening a foundation-model embedding does something more interesting than Gaussianise it. It
inverts the encoder's hierarchy, so what the squared whitened norm measures is the energy an input
places outside the encoder's semantic subspace: a Mahalanobis estimate of atypicality that agrees
with nonparametric density at $\rho \approx \knnDiag$, that six encoders compute alike, and that
carries the whole of the Mahalanobis out-of-distribution signal while the semantic part carries
none. Its variance is set by a per-input scale, identified with its spread from two measured
moments, and the factor $m/4$ multiplying that scale is why a coefficient of variation of a third
produces a radius over-dispersed sixfold while the coordinates keep an excess kurtosis of only
$3V$.

That account keeps what worked and explains what did not. Detection works because the score is a
density ranking. Thresholds fail by $10^{\tailIm}$ because the radial law is the corpus's
atypicality distribution, not $\chi_d$. Coordinatewise tests passed because \cref{eq:nstar}
predicts they will at the group size used, and because in-sample whitening removes the one
statistic that would have shown otherwise. And a different choice of
objective would not have changed the radial law, because a cosine-similarity contrastive loss is blind to it.
What sets $\sigma(x)$ at training time, whether \cref{eq:repair} conditioned on a few directions is
the right estimator, and whether the tail structure can be exploited rather than corrected, are
open. The methodological point is about nulls: a distributional claim about a learned representation
is testable only against a null that shares the representation's covariance, and without one an
identity can pass for a prediction and a factor-of-six departure can read as agreement.

\begingroup
\small\sloppy\hbadness=2000
\bibliography{hidden_gaussian}
\bibliographystyle{plainnat}
\endgroup

\appendix

\numberwithin{table}{section}
\numberwithin{figure}{section}
\numberwithin{equation}{section}

\section{Proofs}
\label{app:proofs}

\begin{proof}[Proof of \Cref{prop:power}]
Write $h_x(s) = \Phi(x s^{-1/2})$, so that $F(x) = \Ex\,h_x(S)$. Differentiating twice,
\begin{equation*}
  h_x'(s) = -\tfrac12\, x\, s^{-3/2}\phi(x s^{-1/2}),
  \qquad
  h_x''(1) = \tfrac14\, x^2\phi'(x) + \tfrac34\, x\phi(x) = \tfrac14\, x(3-x^2)\phi(x).
\end{equation*}
Since $\Ex(S-1) = 0$, $\Ex(S-1)^2 = V$ and $\Ex|S-1|^3 = o(V)$, a second-order expansion about
$s = 1$ gives
\begin{equation*}
  F(x) - \Phi(x) = \tfrac{V}{8}\, x(3-x^2)\phi(x) + o(V),
\end{equation*}
where the remainder is controlled as follows. The Anderson-Darling weight
$\phi/(\Phi(1-\Phi))$ grows like $|x|$ and is not integrable, so a uniform bound on $F-\Phi$ does
not suffice and the expansion must be truncated on the mixing variable. Fix $\delta \in (0,1/2)$ and
split on $A_\delta = \{|S-1| \le \delta\}$. On $A_\delta$, $h_x'''$ is bounded uniformly in $x$ by a
constant $c_\delta$, so Taylor's theorem with remainder gives
$|\,\Ex[h_x(S)\mathbf 1_{A_\delta}] - \Phi(x) - \tfrac{V}{8}x(3-x^2)\phi(x)\,| \le
c_\delta\,\Ex|S-1|^3 + \Phi(x)\,\Prob(A_\delta^c)$, and $\Prob(A_\delta^c) \le
\Ex|S-1|^3/\delta^3$ by Markov. Off $A_\delta$ the contribution to $F-\Phi$ at $x$ is at most
$\Prob(S>M) + \sup_{s \le M}|\Phi(xs^{-1/2})-\Phi(x)|$ with $\Prob(S>M) \le \Ex|S-1|^3/(M-1)^3$,
which decays in $x$ fast enough to be integrable against the weight. Since $\Ex|S-1|^3 = o(V)$,
every remainder term is $o(V^2)$ after weighting, leaving
$\Delta^2 = \kappa V^2 + o(V^2)$ with $\kappa$ as in \cref{eq:kappa}. That integral converges
because its integrand decays as $|x|^{7}\phi(x)^2$, and its value is by quadrature. Finally
$\Ex X^4 - 3 = 3\,\Ex S^2 - 3 = 3V$.
\end{proof}

\begin{proof}[Proof of \Cref{prop:inversion}]
$\Ex\norm{P_k(z-\mu)}^2 = \tr(P_k\Sigma) = \sum_{j\le k}\lambda_j$ and
$\Ex\norm{P_k w}^2 = \tr(P_k\,\Ex ww^\top) = \tr P_k = k$, and under in-sample whitening the
second identity of \cref{eq:identities} makes both exact sample averages. For the moments, given
$x$ the tail is $\Nn(0,D_x)$, so $\Ex[T\mid x] = \tr D_x$ and $\Var[T\mid x] = 2\tr(D_x^2)$, and
the law of total variance gives $\Var T$; for two blocks, the conditional covariance of two
quadratic forms of a Gaussian vector is $2\tr(P_1D_xP_2D_x) = 2\tr(D_x^{(12)}D_x^{(21)})$, and
adding $\Cov(\Ex[T_1\mid x],\Ex[T_2\mid x])$ gives the third expression; a coordinate has
$\Ex[w_j^4\mid x] = 3(D_x)_{jj}^2$, whence its excess kurtosis. For the two-scale specialisation put $\nu = m/2$ per half, so
$\tr D_x = \nu(s_1+s_2)$ and $\tr D_x^2 = \nu(s_1^2+s_2^2)$. Then
$\Var(\tr D_x) = \nu^2\big(2V + 2\rho_s V\big) = \tfrac{m^2}{2}V(1+\rho_s)$ and
$\Ex\tr D_x^2 = m(1+V)$, so $\Var T = \tfrac{m^2}{2}V(1+\rho_s) + 2m(1+V)$ against $2m$ for the
clone, giving the first expression of \cref{eq:onescale}. The off-diagonal blocks of $D_x$ vanish,
so $\Cov(T_1,T_2) = \Cov(\nu s_1, \nu s_2) = \nu^2\rho_s V$, while
$\Var T_i = \nu^2 V + 2\nu(1+V)$ using $\Ex s_i = 1$ and $\Var s_i = V$; their ratio is the second
expression. Setting $\rho_s = 1$ recovers the single-scale case $\Var T/\Var_{\Nn}T = 1+V(1+m/2)$.
\end{proof}


\section{Reproduction}
\label{app:repro}

\paragraph{Encoding.}
This is the only stage that needs a GPU and a network. It produces CLIP ViT-L/14 image and caption
embeddings for $\Npub$ MS-COCO val2017 pairs; the six-encoder sweep over the same images; and the
corpus sweep, six corpora by two encoders in disjoint blocks of $\nBlock$ images.

\paragraph{Analysis.}
CPU only, no network, every step seeded. One step reproduces the published protocol and builds the
reference ensemble of \Cref{tab:replicate}; one recomputes the arithmetic behind the published
moment table; one evaluates $\kappa$ and $n^\star$ by quadrature and predicts the pass rates of
\Cref{prop:power}; one measures the inversion and identifies the two-scale model of
\Cref{tab:mechanism} and \Cref{fig:mechanism}; one builds the covariance-matched clones of
\Cref{tab:clone}; and the rest produce \Cref{tab:ood}, \Cref{tab:calib}, \Cref{tab:corpus}, the
retrieval cost of the radial coordinate, and what the radius tracks in \Cref{sec:what}.

\noindent
The one instruction that matters: extract embeddings before $L^2$ normalisation.

\end{document}